\documentclass[12pt, reqno]{amsart}
\usepackage[margin=3.5 cm]{geometry}
\usepackage{graphicx} 
\usepackage{color}
\usepackage{amssymb}
\usepackage{tabularx}
\usepackage{mathtools}
\newtheorem{theorem}{Theorem}[section]
\newtheorem*{theorem*}{Theorem}          
\usepackage{ amssymb }
\newtheorem*{prop*}{Proposition}  
\newtheorem{coro}{Corollary}[section]
\newtheorem{lemma}[theorem]{Lemma}

\theoremstyle{definition}
\newtheorem{definition}[theorem]{Definition}

\theoremstyle{remark}

\title{Stronger Coreset Bounds for Kernel Density Estimators via Chaining}
\author{Rainie Bozzai and Thomas Rothvoss}
\thanks{This material is based upon work supported by the National Science Foundation Graduate Research Fellowship under Grant No. DGE-2140004. R.B. is also partially supported by Grant No. G20221001-4371 in Aid of Research from Sigma Xi, The Scientific Research Honor Society. T.R. is supported by NSF grant 2318620 and a David \& Lucile Packard Foundation Fellowship.}
\date{\today}
\begin{document}
\maketitle

\begin{abstract}
We apply the discrepancy method and a chaining approach to give improved bounds on the coreset complexity of a wide class of kernel functions. Our results give randomized polynomial time algorithms to produce coresets of size $O\big(\frac{\sqrt{d}}{\varepsilon}\sqrt{\log\log \frac{1}{\varepsilon}}\big)$ for the Gaussian and Laplacian kernels in the case that the data set is uniformly bounded, an improvement that was not possible with previous techniques. We also obtain coresets of size $O\big(\frac{1}{\varepsilon}\sqrt{\log\log \frac{1}{\varepsilon}}\big)$ for the Laplacian kernel for $d$ constant. Finally, we give the best known bounds of $O\big(\frac{\sqrt{d}}{\varepsilon}\sqrt{\log(2\max\{1,\alpha\})}\big)$ on the coreset complexity of the exponential, Hellinger, and JS Kernels, where $1/\alpha$ is the bandwidth parameter of the kernel. 
\end{abstract}
\section{Introduction}

Kernel density estimators provide a non-parametric method for estimating probability distributions. In contrast to parametric machine learning models, in which a set of training data is used to determine an optimal parameter vector which can be used to predict future outcomes without the training data, 
for non-parametric methods the number of parameters of the model has the capability to grow as the size of the dataset grows. Our paper will focus on the task of kernel density estimation, in which one estimates a probability distribution using kernel functions of the form $\mathcal{K}:\mathcal{D}\times\mathcal{D}\rightarrow \mathbb{R}$. Kernels are often defined with a \textit{bandwidth parameter} $1/\alpha$ that controls the width of the kernel function. 

Given a probability distribution $\rho$ and a collection of points $X=\{x_1,...,x_n\}\sim \rho$ sampled independently, it is well known that for certain well-behaved kernel functions $K$, the distribution $\rho$ can be approximated very well by the \textit{kernel density estimator} (KDE)
\begin{equation*}
    \mathrm{KDE}_X(y)=\frac{1}{n}\sum_{i\in[n]}K(x_i,y).
\end{equation*}
 In particular, it is known that under certain conditions on the kernel function, $\mathrm{KDE}_X$ approximates $\rho$ at the minimax optimal rate as $|X|\rightarrow\infty$
\cite{tsybakov2008introduction}.

Although this is an elegant theoretical result, in practice it is computationally inefficient to store and make computations with an arbitrarily large number of data points $n$. One solution to reduce the computational complexity is to use an $\varepsilon$-\textit{coreset} for a kernel density estimator.

\begin{definition}[KDE $\varepsilon$-coreset]
    For fixed $\varepsilon>0$, kernel function $K : \mathcal{D} \times \mathcal{D} \to \mathbb{R}$, and data set $X \subseteq \mathcal{D}$, an \emph{$\varepsilon$-coreset} for $K$ is a subset $Q\subseteq X$ so that
    \begin{equation*}
        \|\mathrm{KDE}_X(y)-\mathrm{KDE}_Q(y)\|_\infty=\sup_{y\in \mathcal{D}}\Big|\frac{1}{|X|}\sum_{x\in X}K(x,y)-\frac{1}{|Q|}\sum_{q\in Q}K(q,y)\Big|\leq \varepsilon.
    \end{equation*}
    We will say that the \textit{coreset complexity} of a kernel function $K$ is the minimum possible size of a $\varepsilon$-coreset $Q$ for $K$.
\end{definition}

In general, coreset complexity bounds will depend on $\varepsilon$ and the dimension $d$ of the kernel domain, and they will be independent of the size of the set $X$. These bounds are also often independent of the choice of $X\subseteq \mathcal{D}$, although several previous results and several of our results give an explicit dependence on $X$ that may allow improvement over existing bounds for sufficiently nice data sets (see Section \ref{known}). In particular, several of our bounds will depend on the radius of the set $X$.

For more details about non-parametric methods and kernel density estimation, see for example \cite{tsybakov2008introduction}.

\subsection{The Discrepancy Approach}\label{discrepancy}
One powerful method for proving bounds on the coreset complexity of kernel functions is the \textit{discrepancy approach}. It has also been used in \cite{phillipstainearoptimal, taioptimal, d1phillips,LibKar} and is based on a method for computing range counting coresets \cite{CHAZELLE, d1phillips, BENTLEY, phillipsterrain}. Following the notational conventions of \cite{phillipstainearoptimal}, we make the following definition.
\begin{definition}[Kernel Discrepancy]
    Given a data set $X\subseteq \mathcal{D}$, a kernel $K:\mathcal{D}\times\mathcal{D}\rightarrow \mathbb{R}$, and a coloring $\beta\in\{\pm 1\}^X$, the \textit{kernel discrepancy at a point $y\in\mathcal{D}$} is defined as
    \begin{equation*}
        \mathrm{disc}_K(X,\beta,y):=\Big|\sum_{x\in X}\beta(x)K(x,y)\Big|.
    \end{equation*}
    Then the \textit{kernel discrepancy} can then be defined as
    \begin{equation*}
        \mathrm{disc}_K(n)=\max_{\substack{X \subseteq \mathcal{D}: \\ \ |X|=n}}\min_{\beta\in\{\pm 1\}^X}\max_{y\in \mathcal{D}}\ \mathrm{disc}_K(X,\beta,y).
    \end{equation*}
\end{definition}

We will also use the notation $\mathrm{disc}_K(X)$ to denote the kernel discrepancy with respect to a fixed data set $X$. Bounds on kernel discrepancy can be leveraged to obtain bounds on the coreset complexity for a given kernel $K$ via the following strategy, often called the ``halving trick" \cite{CHAZELLE, BENTLEY, phillipsterrain}. We construct a coreset of $X$ by iteratively removing half of the points in $X$, and we select which half of the points are removed by creating colorings $\beta\in \{\pm 1\}^X$ minimizing the kernel discrepancy and then removing those points assigned $-1$ (in principle, there is no reason to expect that exactly half of the points are assigned $+1$, and half $-1$, but there are standard techniques to overcome this challenge \cite{Ma99}). Indeed, supposing that we have an optimal choice of signs $\beta\in\{\pm 1\}^X$ such that 
\begin{equation*}
    \sup_{y\in \mathcal{D}} \Big| \sum_{x\in X}\beta(x)K(x,y)\Big|\leq f(n),
\end{equation*} then we simply note that, letting $X^+$ be the set of points assigned $+1$ and $X^-$ be the set of points assigned $-1$, then under the assumption that $|X^-|=|X|/2$, for any $y\in \mathcal{D}$,
\begin{align}\label{above}
\begin{split}
    \frac{1}{|X|}\sum_{x\in X}\beta(x)K(x,y)&=\frac{1}{|X|}\sum_{x\in X^+}K(x,y)-\frac{1}{|X|}\sum_{x\in X^-}K(x,y)\\
    &=\frac{1}{|X|}\sum_{x\in X}K(x,y)-\frac{1}{|X|/2}\sum_{x\in X^-}K(x,y).
    \end{split}
\end{align}

Taking a supremum over $y\in \mathcal{D}$, the final line of \eqref{above} is exactly $\mathrm{KDE}_X(y)-\mathrm{KDE}_{X^-}(y)$. Thus, iterating this procedure $t$ times, and denoting the resulting set at iteration $s$ by $X_s$ (with $X_0:=X$), we find that
\begin{equation*}
  \|\mathrm{KDE}_X-\mathrm{KDE}_{X_t}\|_\infty \leq \sum_{s\in[t]}\|\mathrm{KDE}_{X_{s-1}}-\mathrm{KDE}_{X_{s}}\|_\infty\leq \sum_{s\in[t]} \frac{2^{s-1}}{n}f\big(n/2^{s-1}\big).
\end{equation*}
Assuming that the function $f$ grows sufficiently slowly\footnote{It suffices if $f(n) \leq n^c$ for some fixed constant $0<c<1$.}, this sum will be dominated by the final term, which allows us to calculate the size of a coreset yielding error at most $\varepsilon$. Based on this connection, our proofs will focus on bounding the quantity $\mathrm{disc}_K(n)$ for different kernels $K$ (or in some cases $\mathrm{disc}_K(X)$, when we want to account for the geometry of the data set $X$), and then the ``halving trick" can easily be used to determine the corresponding size of the coreset thus obtained.

\subsection{Positive Definite Kernels and Reproducing Kernel Hilbert Spaces}\label{rkhs}

A kernel $K:\mathcal{D}\times \mathcal{D}\rightarrow \mathbb{R}$ is said to be \textit{positive definite} if given any selection of points $x_1,...,x_m\in \mathcal{D}$, the Gram matrix $G$ given by $G_{ij}=K(x_i,x_j)$ is positive definite. The importance of our assumption that our kernels are positive definite rests on the following famous theorem \cite{Aronszajn1950TheoryOR}.

\begin{theorem}[Moore-Aronszajn, 1950] \label{moore}
Let $T$ be a set and $K$ a positive definite function on $T\times T$. Then there is a map $\phi:T\rightarrow \mathcal{H}_K$ to a unique corresponding Hilbert space $\mathcal{H}_K$ so that for any $s,t\in T$,
\begin{equation*}
    K(s,t)=\langle \phi(s),\phi(t)\rangle_{\mathcal{H}_K}.
\end{equation*}
\end{theorem}

The Hilbert space $\mathcal{H}_K$ is called the \textit{reproducing kernel Hilbert space} (RKHS) associated to $K$. Using this theorem, we can make the following definition:
\begin{definition}
    The \textit{kernel distance} associated to a kernel function $K$ is the function
    \begin{equation*}
        D_K(x,y):=\|\phi(x)-\phi(y)\|_{\mathcal{H}_K}.
    \end{equation*}
\end{definition}
Here $\|h\|_{\mathcal{H}_K} = \sqrt{\langle h,h \rangle_{\mathcal{H}_K}}$ is the Euclidean norm in $\mathcal{H}_K$.
In general, it is only true that $D_K$ is a pseudometric on the domain of $K$; however, this is all we will need for our proofs. Almost all of the kernels commonly used in machine learning applications are positive definite.

\subsection{Summary of Known Results}\label{known}

The problem of coreset complexity for a wide variety of kernel functions has been studied for several decades. Early approaches focused on random samples \cite{silverman,scott,phillipsshapes}, but more recently analytic approaches and new algorithms have been discovered, leading to much stronger bounds. We provide a brief summary of these results, which apply a wide range of techniques and apply to various classes of kernels.

Joshi et al. \cite{joshietal} used a sampling technique to prove a bound of $O((1/\varepsilon^2)(d+\log(1/\delta)))$ on the coreset complexity of any centrally symmetric, non-increasing kernel, where $\delta$ is the probability of failure. Fasy et al. \cite{fasyetal} used sampling to prove a different bound of $O((d/\varepsilon^2)\log(d\Delta/\varepsilon\delta))$, where $\Delta:=\alpha \sup_{x,x'\in X}\|x-x'\|_2$ is the diameter of the data set. This result may improve upon that of Joshi et al. in the case that $K(x,x)>1$, and it also applies to the broader class of Lipschitz kernels.

The following collection of results applies to the collection of \emph{characteristic kernels}, a subset of positive definite kernels that satisfy the additional property that the mapping $\phi_K$ into the associated RKHS is injective, which implies that the induced kernel distance $D_K$ is a metric.
This class contains many, though not all, positive definite kernels, with a notable exception being the exponential kernel (see Theorem \ref{expthm}). Again using random sampling, Lopaz-Paz et al. \cite{lopez-pazetal,otherbook} gave a simpler bound of $O((1/\varepsilon^2)\log(1/\delta))$.

Improved results were proved using an iterative technique called \textit{kernel herding} introduced by Chen et al. \cite{chenetal} to solve a closely related problem called \textit{kernel mean approximation}, which is shown to upper bound kernel density approximation in the case of reproducing kernels \cite{chenetal,koksmabound}. Chen et al. proved a bound of $O(1/(\varepsilon r_X))$, where $r_X$ is the largest radius of a ball centered at $\frac{1}{|X|}\sum_{x\in X}\phi_K(x)\in \mathcal{H}_K$ that is completely contained in $\mathrm{conv}\{\phi(x):\ x\in X\}$. This paper claimed that $r_X$ is always a constant greater than $0$; however, Bach et al. \cite{bachetal} gave a new interpretation of the algorithm and argued that although $r_X$ is arbitrarily small for continuous distributions, the constant $1/r_X$ is finite when $X$ is finite. Their interpretation provided a bound of $(1/r_X^2)\log(1/\varepsilon)$. 
Bach et al. also provided a bound of $O(1/\varepsilon^2)$ in the case of \textit{weighted coreset complexity}, where points in $X$ can be assigned a non-negative weight, and this result was later improved to the setting of unweighted coresets \cite{bachetal2}.

Harvey and Samadi \cite{harvey} applied the kernel herding technique to an even more general problem called \textit{general mean approximation in $\mathbb{R}^{d}$} to provide bounds on the coreset complexity of order $O((1/\varepsilon)\sqrt{n}\log^{2.5}(n))$, where $n=|X|$. The dependence on $n$ is introduced by the worst case outcome of manipulating $r_X$ using affine scaling. Locoste-Julien et al. \cite{bachetal2} showed that actually one can take $n=O(1/\varepsilon^2)$ which improves the bound to $O((1/\varepsilon^2)\log^{2.5}(1/\varepsilon))$.

The next collection of bounds applies to \emph{Lipschitz kernels}, that is, kernels where we can bound the \textit{Lipschitz factor} $C_K:=\max_{x,y,z\in \mathcal{D}}\frac{\|K(z,x)-K(z,y)\|_2}{\|x-y\|_2}$ of $K$. In the case that $C_K$ is a small constant, as is the case for most kernels of interest, it is easy to see that taking a $2\varepsilon/(C_K\sqrt{d})$-net $G_\varepsilon$ over the domain of $K$ and mapping each point $x\in X$ to the closest point in $G_\varepsilon$ (with multiplicity) to obtain $X_{G_{\varepsilon}}$, we find that $\sup_{y\in \mathcal{D}}|\mathrm{KDE}_X(y)-\mathrm{KDE}_{X_{G_{\varepsilon}}}(y)|\leq \varepsilon$. Cortes and Scott's work on the sparse kernel mean problem \cite{cortes} combined with the discretization argument above implies a bound of $O((\Delta/\varepsilon)^d)$ on the coreset-complexity, in the case that $\Delta$ is bounded.

\begin{table}[ht]
    \centering
    \begin{tabular}{c|c|c}
        Phillips \cite{d1phillips} & $O((\alpha/\varepsilon)^{2d/(d+2)}\log^{d/(d+2)}(\alpha/\varepsilon))$ & Lipschitz, PD\\
        Phillips and Tai \cite{phillipstainearoptimal} & $\sqrt{d\log n}$ & Lipschitz, PD\\
        Tai \cite{taioptimal}& $2^{O(d)}$&Gaussian\\
        New & $2^{O(d)}\sqrt{\log\log n}$& Laplacian\\
        New & $O(\sqrt{d\log(\mathrm{radius}X+\log n)})$& Gaussian, Laplacian \\
        New & $O(\sqrt{d\log (2\max\{\alpha,1\})})$ & Exponential, JS, Hellinger
    \end{tabular} \vspace{2mm}
    \caption{Results from the Discrepancy Approach}
    \label{tab:my_label}
\end{table}

Also in the setting of Lipschitz kernels, but now using the discrepancy-based approach described in Section \ref{discrepancy}, Phillips \cite{d1phillips} showed a bound of 
\begin{equation*}
O((\alpha/\varepsilon)^{2d/(d+2)}\log^{d/(d+2)}(\alpha/\varepsilon))
\end{equation*} for kernels with a Lipchitz factor $C_K=O(\alpha)$. Using a sorting argument, Phillips also showed in this paper that for $d=1$ one can achieve a coreset of size $O(1/\varepsilon)$, matching a tight lower bound. Note that in general the coreset complexity is always bounded below by $O(1/\varepsilon)$, as can be seen by taking $O(1/\varepsilon)$ points that are spread far apart \cite{d1phillips}. Phillips and Tai \cite{phillipstainearoptimal} improved these results by combining the discretization approach with the discrepancy approach to give a bound of $O((1/\varepsilon)\sqrt{d\log(1/\varepsilon)})$ for any positive definite, Lipschitz kernel with \textit{bounded influence}, a restriction similar to the impact radius condition that we will define. This result applies to a very wide class of kernels, including all of the kernels that we will discuss in our paper, and also the sinc kernel, for which no earlier non-trivial bounds were known. They also provide a lower bound of $\Omega(\sqrt{d}/\varepsilon)$ for $d\in[2,1/\varepsilon^2)$, and a tight lower bound of $O(1/\varepsilon^2)$ in the case that $d\geq 1/\varepsilon^2$, for all shift and rotation invariant kernels that are somewhere-steep. Tai \cite{taioptimal} later proved that for $d$ constant, the Gaussian kernel has coreset complexity $O(1/\varepsilon)$, matching the optimal lower bound in terms of $\varepsilon$. This result supresses an exponential dependence on $d$, but is still interesting for small-dimensional data sets.

Finally, a related but not directly comparable result due to Karnin and Liberty \cite{LibKar} applies to kernels that are analytic functions of the dot product and satisfy the very strong condition that $\sup_{x,x'\in \mathcal{D}}\|x-x'\|_2\leq R_K$, where $R_K$ is a fixed constant determined by the kernel $K$. In this setting, they show a bound of $O(\sqrt{d}/\varepsilon)$ on the coreset complexity. However, as we will see, for kernels such as the Gaussian or Laplacian kernel, one can only assume that $\sup_{x,x'\in \mathcal{D}}\|x-x'\|_2\leq n\log n$, where $n=|X|$, and thus this result does not apply. Their result can however be interpreted to give a bound of $O(\alpha \exp(\alpha)\sqrt{d}/\varepsilon)$ on the coreset complexity of the exponential kernel, as in this case the domain of the kernel function does have constant diameter.

\subsection{Our Contributions}\label{results}

Our results require the following two definitions.
\begin{definition}[Impact Radius]
    Given a kernel $K:\mathcal{D}\times \mathcal{D}\rightarrow \mathbb{R}$ for $\mathcal{D} \subseteq \mathbb{R}^d$, we define the \textit{impact radius} of $K$ as
    \begin{equation*}
        r_K(n):=\inf\big\{r\geq 0:\ \|x-y\|_2\geq r\ \implies\ |K(x,y)|\leq 1/n\ \ \forall x,y\in\mathcal{D}\big\}.
    \end{equation*}
\end{definition}
    \begin{definition}[Query Space]\label{query}
        Given a kernel $K:\mathcal{D}\times \mathcal{D}\rightarrow \mathbb{R}$ and a data set $X\subseteq \mathcal{D}$, we define the \textit{query space of $K$ with respect to $X$} as
    \begin{equation*}
        Q=\Big\{y\in \mathcal{D}:\ \exists x\in X\ \mbox{s.t.}\  \|x-y\|_2 \leq r_K(|X|)\Big\}=\mathcal{D} \cap \Big(\bigcup_{x\in X} (x+r_K(|X|)B_2^d)\Big).
    \end{equation*}
\end{definition}

Note that in general both the impact radius and the query space may depend explicitly on the bandwidth parameter $\alpha$, but this dependence often cancels out, making the bounds obtained independent of $\alpha$. One notable exception where this cancellation does not occur is when the domain $\mathcal{D}$ of the kernel is compact, for example for the exponential, Hellinger, and Jensen-Shannon kernels; we return to this idea in Sections \ref{applications} and \ref{jssection}. 
We also note that in the event that the query space $Q$ given by a particular data set is disconnected, it suffices to consider the largest connected component of $Q$. To see this, note that by the definition of impact radius, the query points and data points from distinct connected components do not interact, thus we can apply the same bound to each connected component independently. For the proof of our results, we will assume that $Q$ is connected.

Our two main results will be in terms of the size of the query space and the impact radius of the kernel, respectively. One of the key challenges in earlier applications of the discrepancy approach was that the domain $\mathcal{D}$ is often $\mathbb{R}^d$, the sphere $S^d$, the standard $(d-1)$-dimensional simplex $\Delta^d$, or some other potentially unbounded and/or uncountably infinite space. These domains make bounding the discrepancy challenging, as probabilistic techniques are often used, and the size of these spaces make the union bound ineffective. The following lemma shows how the query space and impact radius can simplify this problem to at least ensure that the domain is bounded for kernels with sufficiently nice decay properties.

\begin{lemma}\label{qred} Let $K : \mathcal{D} \times \mathcal{D} \to \mathbb{R}$ be a kernel, $X\subseteq \mathcal{D}$ a data set, and $Q$ the query space associated to $K$ and $X$. Then
    \begin{equation*}
      \mathrm{disc}_K(n)\leq \mathrm{disc}_{K|_Q}(n)+O(1).
    \end{equation*}
    \end{lemma}
    \begin{proof}
    To prove the lemma, it suffices to show that for any $y \in \mathcal{D} \setminus Q$, $\mathrm{disc}_K(X,\beta,y)=O(1)$ for all $\beta\in\{\pm 1\}^X$. But this follows immediately, as we know by the definition of $Q$ that for any $y\in \mathcal{D}\setminus Q$ and $x\in X$, $|K(x,y)|\leq 1/n$. Thus
    \begin{equation*}
\mathrm{disc}_K(X,\beta,y)=\Big|\sum_{x\in X}\beta(x)K(x,y)\Big|\leq \sum_{x\in X} |K(x,y)|\leq 1.\qedhere
    \end{equation*}
\end{proof}

Earlier approaches using the discrepancy approach \cite{phillipstainearoptimal,taioptimal} needed an even stronger version of Lemma \ref{qred} for Lipschitz kernels that also ensured that the query space could be made finite up to an $O(1)$ error using the Lipschitz constant of the kernel and a sufficiently small $\varepsilon$-net. This lemma provides an extra factor of $\sqrt{\log n}$ in many results that we will avoid by using \textit{chaining} \cite{vershynin}, a multi-step construction of $\varepsilon$-nets.
Note that in general, for an arbitrary data set, the query space can have volume up to $n\mathrm{Vol}_d(r_KB_2^d)$, as it is possible that the data points are well spread out and thus the impact radii of data points do not intersect. This observation is why \cite{LibKar} cannot be applied to the general Gaussian kernel, for example.

Our first result gives a discrepancy bound that depends explicitly on the choice of the data set $X\subseteq \mathbb{R}^d$. 
\begin{theorem}\label{Thm1}
    Let $K:\mathbb{R}^d\times \mathbb{R}^d\rightarrow \mathbb{R}$ be a kernel with bandwidth parameter $1/\alpha$ and $X\subseteq \mathbb{R}^d$ be a dataset. Denote the query space of $(K,X)$ by $Q$, and define $R=R(X)$ so that $Q\subseteq RB_2^d$. If $K$ satisfies the following properties:
    \begin{enumerate}
        \item[(i)] $K$ is positive definite;
        \item[(ii)] $K$ satisfies $K(x,x)=1$ for all $x\in Q$;
        \item[(iii)]$K=\kappa(\alpha\|x-y\|_2)$, where $\kappa:\mathbb{R}_{\geq 0}\rightarrow [-1,1]$ is strictly decreasing and continuous; and
        \item[(iv)] The following bound holds, where $[0,b_K)$ is the domain of the integrand:
        \begin{equation*}
\int_{0}^{b_K}\sqrt{-\ln\left[\kappa^{-1}\left(\frac{2-r^2}{2}\right)\right]}\mathrm{d}r=O(1);
        \end{equation*}
    \end{enumerate}
    then 
    \begin{equation*}
\mathrm{disc}_K(X)=O(\sqrt{d\log (2\max\{R\alpha,1\})}).
    \end{equation*}
    Moreover, there is a randomized algorithm that on input of $X$ and $K$, finds an according coloring $\beta\in \{ \pm 1\}^X$ in polynomial time.
\end{theorem}

We will see in Section \ref{applications} that Theorem \ref{Thm1} can give strong improvements on the current best known bounds on the coreset complexity for the Gaussian and Laplacian kernels in the case that the data set $X$ has bounded diameter. In particular, we have the following corollary.

\begin{coro}\label{coro1}
    Let $K$ be a kernel satisfying the conditions of Theorem \ref{Thm1}, and let $X\subseteq \mathbb{R}^d$ be any data set. Then
    \begin{equation*}
        \mathrm{disc}_K(X,n)\leq O(\sqrt{d\log (\mathrm{radius}(X)+\kappa^{-1}(1/n))}).
    \end{equation*}
\end{coro}

Note that because of the discretization necessary to prove Phillip and Tai's bounds \cite{phillipstainearoptimal}, even in the case where the data set is bounded uniformly, the $\sqrt{\log n}$ factor in their discrepancy bound cannot be removed.

In the case where we do not want to account for the geometry of the data set $X$, we present the following stronger bound in the case that the dimension $d$ is taken to be constant.

\begin{theorem}\label{Thm2}
Let  $K:\mathbb{R}^d\times \mathbb{R}^d\rightarrow \mathbb{R}$ be a kernel with bandwidth parameter $1/\alpha$ and $X\subseteq \mathbb{R}^d$ a data set. If $K$ satisfies the following properties:
    \begin{enumerate}
        \item[(i)] $K$ is positive definite;
        \item[(ii)] $K$ satisfies $K(x,x)=1$ for all $x\in Q$;
        \item[(iii)] $K=\kappa(\alpha\|x-y\|_2)$, where $\kappa:\mathbb{R}_{\geq 0}\rightarrow [-1,1]$ is strictly decreasing and continuous; and
        \item[(iv)] The following bound holds:
        \begin{equation*}
\int_{0}^{b_K}\sqrt{-\ln\left[\kappa^{-1}\left(\frac{2-r^2}{2}\right)\right]}\mathrm{d}r=O(1);
        \end{equation*}
    \end{enumerate}
    then
    \begin{equation*}
        \mathrm{disc}_K(n)=2^{O(d)}\sqrt{\log (\kappa^{-1}(1/n))}.
    \end{equation*}
     Moreover, there is a randomized algorithm that on input of $X$ and $K$, finds an according coloring $\beta \in \{ \pm 1\}^X$ in polynomial time.
\end{theorem}

 We will see in Section \ref{applications} that Theorem \ref{Thm2} yields significantly stronger bounds than \cite{phillipstainearoptimal} for several important kernels in machine learning, assuming the data is sufficiently low dimensional.

Our last three results demonstrate that in cases where the kernel is not a function of the Euclidean distance, a similar technique can still provide strong bounds. In particular, we give $O(\sqrt{d})$ bounds for the Jensen-Shannon, exponential, and Hellinger kernels, with a logarithmic dependence on the bandwidth parameter $\alpha$. We define $\Delta_d := \{ x \in \mathbb{R}^{d}_{\geq 0} \mid \sum_{i\in[d]} x_i = 1\}$ as the \emph{standard} $(d-1)$-\emph{dimensional simplex} in $\mathbb{R}^{d}$. 

\begin{theorem}\label{jsthm}
    The Jensen-Shannon (JS) kernel
    \begin{equation*}
        K_{JS}:\Delta_d\times \Delta_d\rightarrow \mathbb{R},\ \ \ K_{JS}(x,y)=\exp\left(-\alpha\big(H\big(\tfrac{x+y}{2}\big)-\tfrac{H(x)+H(y)}{2}\big)\right),
    \end{equation*}
    where $H(x)=-\sum_{i\in[d]}x_i\log x_i$ is the Shannon entropy function, has discrepancy satisfying 
    \begin{equation*}
        \mathrm{disc}_{K_{JS}}(n)=O(\sqrt{d\log (2\max\{ \alpha,1\})}).
    \end{equation*}
    Further, the JS Kernel has coreset complexity
    \begin{equation*}
        O(\sqrt{d\log (2\max\{\alpha,1\})}/\varepsilon),
    \end{equation*}
    and such a coreset can be constructed in randomized polynomial time.
\end{theorem}

This result greatly improves on the current best bound of $O(\sqrt{d\log n})$ on the discrepancy of the JS kernel, in particular by dropping all dependence on the size $n$ of the data set. We note that as the JS kernel is not shift or rotation invariant, it is not known to satisfy any lower bounds other than $O(1/\varepsilon)$, which holds for all kernels.

We also have the same bound for the exponential kernel.
\begin{theorem}\label{expthm}
    The exponential kernel 
    \begin{equation*}
        K_e:S^d\times S^d\rightarrow \mathbb{R},\ \ \  K_e(x,y)=\exp(-\alpha(1-\langle x,y\rangle)),
    \end{equation*} has discrepancy satisfying
    \begin{equation*}
        \mathrm{disc}_{K_e}(n)=O(\sqrt{d\log (2\max\{\alpha,1\})}).
    \end{equation*}
Further, the exponential kernel has coreset complexity
    \begin{equation*}
        O(\sqrt{d\log (2\max\{\alpha,1\})}/\varepsilon),
    \end{equation*}
    and such a coreset can be constructed in randomized polynomial time.
\end{theorem}

As we mentioned in Section \ref{known}, Karnin and Liberty's proof technique in \cite{LibKar} gives a bound of $O(\alpha\exp(\alpha)\sqrt{d})$ on this quantity,
hence our analysis gives a doubly-exponential improvement in terms of the bandwidth parameter $\alpha$.

As a corollary of Theorem \ref{expthm} we obtain the same bound for the Hellinger kernel, again giving the best known result for this kernel.

\begin{theorem}\label{hell}
    The Hellinger kernel 
    \begin{equation*}
        K_H:\Delta^d\times \Delta^d\rightarrow \mathbb{R},\ \ \  K_H(x,y)=\exp\left(-\alpha\sum_{i=1}^d\big(\sqrt{x}-\sqrt{y}\big)\right),
    \end{equation*} has discrepancy satisfying
    \begin{equation*}
        \mathrm{disc}_{K_H}(n)=O(\sqrt{d\log (2\max\{\alpha,1\})})
    \end{equation*}
Further, the Hellinger kernel has coreset complexity
    \begin{equation*}
        O(\sqrt{d\log (2\max\{\alpha,1\})}/\varepsilon),
    \end{equation*}
    and such a coreset can be constructed in randomized polynomial time.
\end{theorem}

As for the JS kernel, no better lower bound than $O(1/\varepsilon)$ is known on the coreset complexity of the exponential and Hellinger kernels.

\section{Preliminaries}
In this section we introduce several results from discrepancy theory that will be essential in our proofs.

\subsection{Banaszczyk's Theorem and the Gram-Schmidt Walk}

As we always work with positive definite kernels, by Theorem \ref{moore} each kernel $K$ has an associated RKHS $\mathcal{H}_K$ and a map $\phi_K:\mathcal{D}\rightarrow \mathcal{H}_K$. Having mapped the data set $X$ into $\mathcal{H}_K$ via $\phi_K$, a key step of our analysis will be to apply Banaszczyk's theorem \cite{Ba98}, in its algorithmic form \cite{algban}, to the vectors $\{\phi_K(x):\ x\in X\}$, which we observed above have norm $1$ for any kernel satisfying $K(x,x)=1$ for $x\in \mathcal{D}$. The algorithmic variant of Banaszczyk's Theorem can be stated as follows.

\begin{theorem}[Gram-Schmidt Walk \cite{algban}]\label{ban} There is a polynomial-time randomized algorithm that takes as input vectors $v_1,\ldots,v_n\in \mathbb{R}^m$ of $\ell_2$ norm at most $1$ and outputs random signs $\beta\in\{\pm 1\}^n$ such that the (mean zero) random variable $\sum_{i=1}^n \beta_iv_i$ is $O(1)$-subgaussian.
\end{theorem}

We give one of several equivalent definitions of subgaussian random variables, as well as the definition of the \textit{subgaussian norm}, which we will need in Section \ref{dudleyintro} (see \cite{vershynin} for more details).

\begin{definition}
    A random variable $X$ is \emph{$K$-subgaussian} if the tails of $X$ satisfy
    \begin{equation*}
        \mathbb{P}[|X|\geq t]\leq 2\exp(-t^2/K^2)\ \ \forall t\geq 0.
    \end{equation*}
  The \textit{subgaussian norm} of $X$ is then
  \begin{equation*}
      \|X\|_{\psi_2}:=\inf\{s > 0:\ \mathbb{P}[|X|\geq t]\leq 2\exp(-t^2/s^2)\ \ \forall t\geq 0\}.
  \end{equation*}
\end{definition}

\subsection{Chaining and Dudley's Theorem}\label{dudleyintro}

As we mentioned in Section \ref{results}, one of the key factors contributing to the factor of $\sqrt{\log n}$ in \cite{phillipstainearoptimal} is the necessity of introducing a $1/n$-net on the query space. This net necessitates at least $O(n^d)$ data points in the query space, which is precisely the factor yielding $\sqrt{d\log n}$, as the final bound is proportional to $\sqrt{\log |Q|}$. To avoid this extra factor of $\sqrt{\log(n)}$, we replace the union bound-type argument of \cite{phillipstainearoptimal}
by a \emph{chaining} approach. The term chaining refers to the fact that rather than applying an $\varepsilon$-net to the entire space, we continue to construct finer and finer $\varepsilon$-nets. For more details see \cite{vershynin}.

\begin{definition}
    Given a (pseudo)metric space $(T,d)$ and $r>0$:
    \begin{itemize}
    \item $B_d(s,r)=\{t\in T:\ d(s,t)\leq r\}$;
        \item $\mathrm{diam}(d):=\sup_{t,s\in T}d(t,s)$;
        \item $\mathcal{N}(T,d,r)$ is the size of a minimal $r$-cover of $T$ w.r.t. $d$, i.e.
        \begin{equation*}
         \mathcal{N}(T,d,r)=\min\Big\{|S|:\ S\subseteq T, \ T\subseteq \bigcup_{s\in S}B_d(s,r)\Big\}. 
        \end{equation*}
    \end{itemize}
\end{definition}

\begin{theorem}[Dudley's Integral Inequality]\label{dud}
    Let $(X_t)_{t\in T}$ be a random process on a pseudometric space $(T,d)$ satisfying
    \begin{equation*}
        \|X_t-X_s\|_{\psi_2}\leq d(t,s)\ \ \forall t,s\in T.
    \end{equation*}
    Then for any $t_0 \in T$,
    \begin{equation*}
        \mathbb{E}\sup_{t\in T}|X_t-X_{t_0}|\lesssim \int_{0}^{\mathrm{diam}(d)}\sqrt{\log \mathcal{N}(T,d,r)}\ \mathrm{d}r.
    \end{equation*}
\end{theorem}
For our application we need to control the absolute value which can be done as follows:
\begin{theorem}[Dudley's Integral Inequality II]\label{dud2}
    Let $(X_t)_{t\in T}$ be a random process on a pseudometric space $(T,d)$ satisfying
    \begin{equation*}
        \|X_t-X_s\|_{\psi_2}\leq d(t,s)\ \ \forall t,s\in T.
    \end{equation*}
    Then for any $t_0 \in T$,
    \begin{equation*}
        \mathbb{E}\sup_{t\in T}|X_t|\lesssim \int_{0}^{\mathrm{diam}(d)}\sqrt{\log \mathcal{N}(T,d,r)}\ \mathrm{d}r + \|X_{t_0}\|_{\psi_2}.
    \end{equation*}
\end{theorem}
\begin{proof}
By the triangle inequality $\mathbb{E}\sup_{t\in T}|X_t| \leq \mathbb{E}\sup_{t\in T}|X_t - X_{t_0}| + \mathbb{E}[ |X_{t_0}| ]$. The claim 
then follows from Theorem~\ref{dud} and the fact that $\mathbb{E}[|X_{t_0}|] \lesssim \|X_{t_0}\|_{\psi_2}$.
\end{proof}
There is also a concentration version of this inequality.
\begin{theorem}[Concentration Version of Dudley's Inequality]\label{dudconc} Let $(X_t)_{t\in T}$ be a random process on a pseudometric space $(T,d)$ satisfying
\begin{equation*}
    \|X_t-X_s\|_{\psi_2}\leq d(t,s)\ \forall t,s\in T.
\end{equation*}
Then, for every $u\geq 0$, the event
\begin{equation*}
    \sup_{t,s\in T}|X_t-X_s|\leq C\left[\int_0^{\mathrm{diam}(T)}\sqrt{\log \mathcal{N}(T, d, r)}\ \mathrm{d}r\ +\ u \cdot \mathrm{diam}(T)\right]
\end{equation*}
holds with probability at least $1-2\exp(-u^2)$, where $C>0$ is a universal constant. 
\end{theorem}

We will see that condition (iv) in Theorems \ref{Thm1} and \ref{Thm2} is exactly the quantity that will appear in the integral of interest once we make a comparison between the pseudometric $D_K$ and the Euclidean distance. For this comparison we will also need the following facts about covering numbers \cite{AGAPart1}.

\begin{lemma}\label{Euclideancovering}
Let $N(A,B)$ denote the number of copies of body $B$ needed to cover $A$. For any convex set $K\subseteq \mathbb{R}^d$ and $r>0$, one has 
\begin{equation*}
    N(K,rB_2^d)\leq 2^d\frac{\mathrm{Vol}_d(K+\tfrac{r}{2}B_2^d)}{\mathrm{Vol}_d(rB_2^d)}.
\end{equation*}
\end{lemma}
\begin{lemma}\label{coveringOfConvexBody}
For any symmetric convex body $P\subseteq \mathbb{R}^d$ and $r>0$, one has 
\begin{equation*}
    N(P,rP)\leq \Big(1+\frac{2}{r}\Big)^d.
\end{equation*}
\end{lemma}

\section{Proofs of Theorems \ref{Thm1} and \ref{Thm2}}

We begin with the proof of Theorem \ref{Thm1}. 

\begin{proof}[Proof of Theorem \ref{Thm1}]
We may assume that $R\alpha \geq 1$, otherwise replace $R$ by $\frac{1}{\alpha}$.
Let $K:\mathbb{R}^d\times \mathbb{R}^d\rightarrow\mathbb{R}$ and $X\subseteq \mathbb{R}^d$ satisfy conditions (i)-(iv) outlined in Theorem \ref{Thm1}. As $K$ is positive definite, there exists an RKHS $\mathcal{H}_K$ and a map $\phi:\mathbb{R}^d\rightarrow \mathcal{H}_K$ so that
\begin{equation*}
    K(x,y)=\langle \phi(x),\phi(y)\rangle_{\mathcal{H}_K} \ \ \forall x,y\in \mathbb{R}^d.
\end{equation*}

We apply the Gram-Schmidt walk from Theorem \ref{ban} to the vectors $\phi(x)$ for $x\in X$, noting that by condition (ii), $\|\phi(x)\|_{\mathcal{H}_K}=1$ for each $x\in X$. 
The algorithm yields a distribution $\mathcal{P}$ over $\{\pm 1\}^X$ so that the random variable $\Sigma:=\sum_{x\in X}\beta(x)\phi(x)$, with $\beta\sim \mathcal{P}$, is $O(1)$-subgaussian. In particular, as we know that $\|\phi(y)\|_{\mathcal{H}_K}=1$ for any $y\in Q$, the (mean zero) random variable
\begin{equation*}
   \Sigma_y:= \big\langle\Sigma, \phi(y)\big\rangle=\sum_{x\in X}\beta(x)K(x,y)=\mathrm{disc}_K(X,\beta,y)
\end{equation*}
is $O(1)$-subgaussian. We will apply Theorem \ref{dud} to the mean zero process $(\Sigma_y)_{y\in Q}$ with the pseudometric $D_K(x,y)=\|\phi(x)-\phi(y)\|_{\mathcal{H}_K}$, the kernel distance defined in Section \ref{rkhs}. To see why $D_K$ satisfies the condition of Theorem \ref{dud}, note that for any $y,q\in Q$,
\begin{equation*}
    \mathrm{Var}[\Sigma_y-\Sigma_q]^{1/2}=\mathbb{E}[\big\langle \Sigma, \phi(y)-\phi(q)\big\rangle^2]^{1/2} \lesssim 
    \|\phi(y)-\phi(q)\|_{\mathcal{H}_K}=D_K(y,q).
\end{equation*}

Thus by Theorem \ref{dud2},
\begin{equation}\label{dudapp}
   \mathbb{E}\  \mathrm{disc}_{K|_Q}(X)=\mathbb{E}\sup_{y\in Q} |\Sigma_y|\lesssim \int_0^{\mathrm{diam}(D_K)}\sqrt{\log \mathcal{N}(Q,D_K,r)}\ \mathrm{d}r + \|\Sigma_{y_0}\|_{\psi_2}
\end{equation}
for any fixed $y_0 \in Q$. Here $\|\Sigma_{y_0}\|_{\psi_2} \leq O(1)$, and so we may ignore this lower order term.
To estimate $\mathcal{N}(Q,D_K,r)$, we use conditions (ii) and (iii) to note that
\begin{equation*}
    D_K(q,y)=\sqrt{2-2\kappa(\alpha\|q-y\|_2)}\leq r\ \ \iff\ \ \|q-y\|_2\leq \frac{1}{\alpha}\kappa^{-1}\left(\frac{2-r^2}{2}\right),
\end{equation*}
    where we use that $\kappa$ is strictly decreasing. From this calculation we conclude that 
    \begin{equation*}\label{l2compare}
        \mathcal{N}(Q, D_K, r)=\mathcal{N}\left(Q, \|\cdot\|_2, \frac{1}{\alpha}\kappa^{-1}\left(\frac{2-r^2}{2}\right)\right).
    \end{equation*}
    Taking $c:=\frac{1}{\alpha}\kappa^{-1}\left(\frac{2-r^2}{2}\right)$ for a moment, we can bound the quantity on the right using Lemma \ref{Euclideancovering}:
    \begin{align}\label{volapprox}
    \begin{split}
        \mathcal{N}\left(Q, \|\cdot\|_2,c\right)&=N(Q, \ cB_2^d)\\
        &\leq N(RB_2^d, cB_2^d)\\
        &\leq 2^d\frac{\textrm{Vol}_d((R+\frac{c}{2})B_2^d)}{\textrm{Vol}_d(c B_2^d)}\\
        &\leq \left(\frac{4R}{c}\right)^d,
        \end{split}
    \end{align}
    where we use Lemma~\ref{Euclideancovering} and $c \leq R$ (if $c>R$ then the covering number is trivially 1). 
    Using this bound in \eqref{dudapp}, we find
    \begin{align}\label{alphacalc}\begin{split}
    \mathbb{E}\ \mathrm{disc}_{K|_Q}(X)&\lesssim \int_0^{\mathrm{diam}(D_K)}\sqrt{\log\left[\left(\frac{4R\alpha}{{\kappa^{-1}\left(\frac{2-r^2}{2}\right)}}\right)^d\right]}\ \mathrm{d}r   \\
   &\lesssim \sqrt{d}\int_0^{\mathrm{diam}(D_K)}\sqrt{\log\left[\frac{4R\alpha}{{\kappa^{-1}\left(\frac{2-r^2}{2}\right)}}\right]}\ \mathrm{d}r \\
   &\lesssim \sqrt{d\log (2R\alpha)}+\sqrt{d}\int_0^{\mathrm{diam}(D_K)}\sqrt{\log\left[\frac{1}{{\kappa^{-1}\left(\frac{2-r^2}{2}\right)}}\right]}\ \mathrm{d}r \\
   & \lesssim \sqrt{d\log (2R\alpha)} + O(\sqrt{d}).
       \end{split}
   \end{align}
   To justify these calculations, note that in the event that the quantity inside the radical is negative, it means that 
   \begin{equation*}
       4R\alpha \leq \kappa^{-1}\Big(\frac{2-r^2}{2}\Big),
   \end{equation*}
   and hence by \eqref{volapprox} 
   we know that $\mathcal{N}(Q,D_K,r)=2^{O(d)}$. As $\mathrm{diam}(D_K)\leq 2$, 
   the domain of $r$ values for which this occurs has length $\leq 2$, so we can simply restrict the domain to this point while losing at most an additive $O(\sqrt{d})$. 
   Note that as $\kappa$ is assumed strictly decreasing and continuous, $\kappa^{-1}\big(\frac{2-r^2}{2}\big)$ is an increasing, continuous function of $r$. Thus in particular $\kappa^{-1}\big(\frac{2-r^2}{2}\big)\rightarrow 0$ as $r\rightarrow 0$, based on our assumption that $K(x,x)=1$ for all $x\in \mathbb{R}^d$. 
   Finally, by Lemma \ref{qred},
   \begin{equation*}
       \mathbb{E}\ \mathrm{disc}_{K}(X)\leq \mathbb{E}\ \mathrm{disc}_{K|_Q}(X)+O(1)\lesssim \sqrt{d\log (2\alpha R)}.\qedhere
   \end{equation*}
\end{proof}

Next we prove Theorem \ref{Thm2}, which gives improved bounds independent of the geometry of the data set $X$.
\begin{proof}[Proof of Theorem \ref{Thm2}]
Fix a kernel $K$ satisfying conditions (i)-(iv) with impact radius $r_K$, and let $X\subseteq \mathbb{R}^d$. Fix the associated query space $Q$ as given by Definition \ref{query}. As we assume that $K$ is positive definite and satisfies $K(x,x)=1$ for all $x\in Q$, there exists a map $\phi: \mathbb{R}^d\rightarrow \mathcal{H}_{K}$, where $\mathcal{H}_K$ is a RKHS such that $K(x,y)=\langle \phi(x),\phi(y)\rangle$ for all $x,y\in \mathbb{R}^d$, and $\|\phi(x)\|_{\mathcal{H}_K}=1$ for all $x\in \mathbb{R}^d$. We first find a maximal set of points $q_1,\ldots,q_m\in Q$ such that for any $\|q_i-q_j\|_2\geq r_K$ for any $i,j\in [m]$. We then partition the space $Q$ into disjoint cells
\begin{equation*}
    Q:=R_1\dot{\cup} \cdots\dot{\cup} R_m
\end{equation*}
so that
\begin{equation}\label{rj}
    R_i\subseteq q_i+r_KB_2^d
\end{equation} for each $i\in[m]$. These cells also partition the set $X$ into smaller sets $X_i:=X\cap R_i$. We begin with the set $X_1$; because the conditions in Theorems \ref{Thm1} and \ref{Thm2} are the same, by identical arguments to the proof of Theorem \ref{Thm1} we obtain a random variable $\Sigma^1:=\sum_{x\in X_1}\beta^{(1)}(x)\phi(x)$, where $\beta^{(1)}\sim \{\pm 1\}^{X_1}$ is drawn according to the Gram-Schmidt walk. By Theorem \ref{ban}, $\Sigma^1$ is $O(1)$-subgaussian, so in particular, for any $y\in Q$, as $\|\phi(y)\|_{\mathcal{H}_K}=1$, the random variable
\begin{equation*}
    \Sigma^1(y):=\Big\langle \phi(y),\sum_{x\in X_1}\beta^{(1)}(x) \cdot \phi(x)\Big\rangle
\end{equation*}
is also $O(1)$-subgaussian. 

For each $j\in[m]$, we now have the collection of (mean zero) random variables $\{\Sigma^j(y):\ y\in R_j\}$. 
We fix $j\in[m]$; then by our assumption in \eqref{rj}, our query space $R_j$ has volume at most $\mathrm{Vol}_d(r_KB_2^d)$. By a calculation analogous to that in \eqref{volapprox}, we have that for all $r>0$,
\begin{equation}\label{coveringbd}
    \mathcal{N}(R_j, D_K, r)\leq \max\Big\{ \Big(\frac{4\alpha r_K}{r}\Big)^d, 1\Big\}.
\end{equation}

Here we note that because of our assumption that $K(x,y)=\kappa(\alpha\|x-y\|_2)$ and the definition of $r_K(n)$, we have for $x,y$ with $\|x-y\|_2 = r_K(n)$,
\begin{equation}\label{kappa}
    K(x,y) = \kappa(\alpha r_K(n))\leq 1/n\ \implies\ \alpha r_K(n)\leq \kappa^{-1}(1/n)\ \implies\ r_K(n)\leq \frac{1}{\alpha}\kappa^{-1}(1/n).
\end{equation}
Thus the quantity $\alpha r_K$ appearing in \eqref{coveringbd} can actually be bounded by a term independent of $\alpha$. Note that by assumption (iii), $\kappa^{-1}$ is a decreasing function. Thus by the same analysis in the proof of Theorem~\ref{Thm2}, we have
\begin{equation*}
 \mathrm{disc}_{K|_{R_j}}(X_1)\lesssim   \int_0^{\mathrm{diam}(D_K|_{R_j})}\sqrt{\log \mathcal{N}(R_j,D_K,r)\ }\mathrm{d}r=O(\sqrt{d\log \kappa^{-1}(1/n)}).
\end{equation*}

Applying Theorem \ref{dudconc} with $u:=C_0\sqrt{d}$ yields that
\begin{equation}\label{failureprob}
    \mathbb{P}\big\{\sup_{y\in R_j}\ |\Sigma^1(y)|\geq C \cdot C_0\sqrt{d}\big\}\leq 2e^{-dC_0^2},
\end{equation}
where $C$ is the constant from Theorem~\ref{dudconc} and $C_0 > 0$.
Moreover, by standard packing arguments, this probability is in fact $0$ for all but at most $2^{O(d)}$ choices of $j\in[m]$, as any $R_j$ not adjacent to (or the same as) $R_1$ will have distance $\Omega(r_K)$ from any point $x\in X_1$, hence by the definition of the impact radius $r_K$ and the bound in \eqref{kappa}, for any such $R_j$ the above probability is $0$. We apply the union bound over the $2^{O(d)}$ cells where \eqref{failureprob} could fail and obtain that for $C_0$ large enough,
\begin{align}\label{goodprob}\begin{split}
    \mathbb{P} \Big\{\forall j\in[m]:\ \sup_{y\in R_j}\ |\Sigma^1(y)|\geq C \cdot C_0\sqrt{d\log r_K}\Big\}&=\mathbb{P}\Big \{\sup_{y\in Q} |\Sigma^1(y)|\geq C_0\sqrt{d\log r_K}\Big\}\\
    &\leq 2^{O(d)}\cdot 2(\kappa^{-1}(1/n))^{-C_0^2d}<1.
    \end{split}
\end{align}

Thus with positive probability, $\mathrm{disc}_K(X_1)\leq C \cdot C_0\sqrt{d\log \kappa^{-1}(1/n)}$. We then fix such an outcome and repeat this construction independently for $R_2,...,R_m$, at each step repeating the Gram-Schmidt algorithm until we get the outcome as in \eqref{goodprob}. After repeating this construction $m$ times, we have a choice of signs $\beta=(\beta^{(1)},...,\beta^{(m)})\in \{\pm 1\}^X$ so that by the triangle inequality
\begin{equation}\label{tri}
    \mathrm{disc}_K(X)\leq \sum_{i\in[m]}\mathrm{disc}_K(X_i).
\end{equation}
However, for our purpose this bound is too weak.
But because the function $K=\kappa(\|x-y\|_2)$ is symmetric in $x$ and $y$, by the same packing argument we made above, we know that 
for each fixed point $y \in Q$, only $2^{O(d)}$ terms in the summand in \eqref{tri} can contribute a discrepancy larger than $O(1)$. Thus we conclude that
\begin{equation*}
\mathrm{disc}_K(X)\leq 2^{O(d)}\sqrt{d\log \kappa^{-1}(1/n)}.
\end{equation*}
As $X\subseteq \mathbb{R}^d$ was arbitrary, the bound follows.
\end{proof}

\section{Applications to Specific Kernels of Interest}\label{applications}

In this section we highlight the applications of Theorems \ref{Thm1} and \ref{Thm2} to several kernels of interest in machine learning.

\subsection{The Gaussian and Laplacian Kernels}

The Gaussian and Laplacian kernels are two of the most important and most commonly used kernels in machine learning \cite{scott}.
The Gaussian kernel is defined by
\begin{equation*}
    K_G:\ \mathbb{R}^d\times \mathbb{R}^d\rightarrow \mathbb{R}\ \ \ \ K_G(x,y)=\exp(-\alpha^2\|x-y\|_2^2),
\end{equation*}
and the Laplacian kernel is defined by
\begin{equation*}
    K_L:\ \mathbb{R}^d\times\mathbb{R}^d\rightarrow \mathbb{R}\ \ \ K_L(x,y)=\exp(-\alpha \|x-y\|_2),
\end{equation*}
where again $1/\alpha>0$ is the bandwidth parameter of the kernels. Currently the best known bound for both the Gaussian and the Laplacian kernels in arbitrary dimension $d$ is $O(\sqrt{d\log n})$, and because these bounds rely on taking $1/n$-nets of the query space, even if the given data set $X$ satisfies nice properties such as boundedness, the $\sqrt{\log n}$ term remains. However, Theorem \ref{Thm1} allows us to give significant improvements on this $O(\sqrt{d\log n})$ bound given such a data set.

\begin{coro}\label{bddgauss}
    For the Gaussian kernel $K_G$ with any bandwidth parameter $\alpha>0$ and any data set $X\subseteq RB_2^d$ for any fixed constant $R>0$, we have
    \begin{equation*}
        \mathrm{disc}_{K_G}(X)=O(\sqrt{d\log\log n}).
    \end{equation*}
    In particular, one can find (in randomized polynomial time) a coreset for $X$ of size
    \begin{equation*}
        O\left(\frac{\sqrt{d}}{\varepsilon}\sqrt{\log\log \frac{1}{\varepsilon}}\right).
    \end{equation*}
\end{coro}
\begin{proof}
We first check the conditions of Theorem \ref{Thm1} for the Gaussian kernel. It is well-known that $K_G$ is positive definite, and clearly $K_G(x,x)=e^0=1$ for any $x\in \mathbb{R}^d$. Taking $\kappa_G(z):=e^{-z^2}$ gives us $K_G(x,y)=\kappa_G(\alpha\|x-y\|_2)$, and $\kappa$ is strictly decreasing on $\mathbb{R}^+$; finally, 
\begin{equation*}
   \sqrt{-\log\big[\kappa^{-1}\big(\tfrac{2-r^2}{2}\big)\big]}\lesssim \sqrt{\log{\left[\frac{1}{\log\big(\tfrac{2}{2-r^2}\big)}\right]}}
\end{equation*}
is dominated by $O(1/r^2)$ for $r\in \mathbb{R}^+$, hence is integrable on its domain. Finally, we note that the query space $Q$ is given by
\begin{equation*}
    Q:=\bigcup_{x\in X}B_2^d(x, r_{K_G})\subseteq O(r_{K_G})B_2^d,
\end{equation*}
as we assume that $\mathrm{radius}(X)=O(1)$. Thus Theorem \ref{Thm1} yields
\begin{equation*}
    \mathrm{disc}_{K_G}(X)\lesssim O(\sqrt{d\log R\alpha})=O(\sqrt{d\log\log n})
\end{equation*}
by the fact that $R=O(r_K)$ and the same observation as in the proof of Theorem \ref{Thm2} that 
\begin{equation*}
    \alpha r_{K_G}\leq \kappa_G^{-1}(1/n)=\sqrt{\log n}.
\end{equation*}



To see the bound on the size of a minimal coreset, we will apply the ``halving trick" shown in Section \ref{discrepancy}. The above argument shows a bound of $f(n)=O(\sqrt{d\log\log n})$ on the kernel discrepancy; iterating the discrepancy calculation $t$ times and removing the points colored $-1$ at each step, we obtain the following bound after $t$ iterations:
\begin{equation*}
    \|\mathrm{KDE}_X-\mathrm{KDE}_{X_t}\|_\infty\lesssim \sum_{s\in[t]}\frac{2^{s-1}}{n}\sqrt{d\log\log\tfrac{n}{2^{s-1}}}.
\end{equation*}
This sum is dominated by the $t^{th}$ term, so we can repeat this calculation until the size of the remaining set of data points $m:=n/2^{s-1}$ satisfies
\begin{equation*}
m\sqrt{d\log\log m}=\varepsilon,
\end{equation*}
which occurs for $m=O\big(\frac{\sqrt{d}}{\varepsilon}\sqrt{\log \log \frac{1}{\varepsilon}}\big)$, yielding the bound.
\end{proof}

This exact technique can always be used to provide coreset bounds, assuming that the term $f(n)$ grows significantly slowly with $n$, which it always will for our applications. Thus we will not repeat this calculation for the proofs of the remaining theorems. Note that an identical argument for the discrepancy yields the proof of Corollary \ref{coro1} stated in Section \ref{results}. However, in general if the diameter of the set $X$ depends on the number of data points, one has to be more careful to obtain bounds on the coreset complexity, as this dependence may change as we remove data points.

The following theorem follows from essentially identical arguments to those in the previous proof, noting that the Laplacian kernel can be written as $K_L(x,y)=\kappa_L(\alpha\|x-y\|_2)$, with $\kappa_L(z)=e^{-z}$.
\begin{coro}\label{boundedlap}
    For the Laplacian kernel $K_L$ with any $\alpha>0$ and any data set $X\subseteq R B_2^d$ for any fixed constant $R>0$ we have
    \begin{equation*}
        \mathrm{disc}_{K_L}(X)=O(\sqrt{d\log\log n}).
    \end{equation*}
    In particular, one can find (in randomized polynomial time) a coreset for $X$ of size
    \begin{equation*}
        O\left(\frac{\sqrt{d}}{\varepsilon}\sqrt{\log\log \frac{1}{\varepsilon}}\right).    \end{equation*}
\end{coro}

Thus in the interest of applications, we can invoke useful properties of the data set $X$ in order to obtain better bounds, something that was not possible due to the discretization technique used in \cite{phillipstainearoptimal}. Karnin and Liberty \cite{LibKar} provide bounds assuming that the entire query space lies in a bounded region (for a particular constant depending on the kernel), but our assumption is much weaker: we are only interested in the geometry of the data set, not the query space.

In the case when our data is sufficiently low dimensional that a $2^{O(d)}$ factor is not prohibitively large, i.e. for $d$ constant, Theorem \ref{Thm2} yields the following improvement on the current best known bound for the Laplacian kernel \cite{phillipstainearoptimal}. 

\begin{coro}
    For the Laplacian kernel $K_L$
 with any $\alpha>0$, we have  \begin{equation*}
    \mathrm{disc}_{K_L}(n)=2^{O(d)}\sqrt{\log\log n}.
    \end{equation*}
    In particular, given any data set $X\subseteq \mathbb{R}^d$, we can find (in randomized polynomial time) a coreset of size 
    \begin{equation*}
     \frac{2^{O(d)}}{\varepsilon}\sqrt{\log\log\frac{1}{\varepsilon}}.   
    \end{equation*}
\end{coro}
\begin{proof}
    The verification that the Laplacian kernel satisfies conditions (i)-(iv) of Theorem \ref{Thm2} is essentially identical to showing that the Gaussian kernel satisfies these properties, so we refer back to the proof of Corollary \ref{bddgauss}. As mentioned previously, the Laplacian kernel can be written as $K_L(x,y)=\kappa_L(\alpha \|x-y\|_2)$, where $\kappa_L(z)=e^{-z}$, and applying Theorem \ref{Thm2} then yields the bound above.
\end{proof}

We note that the same bound follows for the Gaussian kernel as well, though a recent paper of Tai \cite{taioptimal} gives a bound of $O(1)$ on the discrepancy of the Gaussian kernel for $d$ constant. The technique Tai used for the Gaussian does not appear to generalize to other kernels, whereas our technique applies to a broader class of kernels that also includes the Laplacian kernel.

\subsection{The Exponential and the Hellinger kernel}

We can also obtain improved bounds for the exponential kernel
\begin{equation*}
    K_e:S^d\times S^d\rightarrow \mathbb{R}^d,\ \ \ K_e(x,y)=\exp(-\alpha(1-\langle x,y\rangle)),
\end{equation*}
and the Hellinger kernel
\begin{equation*}
    K_H:\Delta_d\times \Delta_d \rightarrow \mathbb{R},\ \ \ \ K_H(x,y)=\exp\Big(-\alpha \sum_{i\in[d]}(\sqrt{x_i}-\sqrt{y_i})^2\Big)
\end{equation*}
by applying Theorem \ref{Thm1}.

\begin{proof}[Proof of Theorem \ref{expthm}]
    First we note that $x,y\in S^d$, hence in particular $\|x\|_2=\|y\|_2=1$, so we can re-write the exponential kernel as
    \begin{equation*}
        K_e(x,y)=\exp\Big(-\frac{\alpha}{2}\|x-y\|_2^2\Big)=K_G|_{S^d},
    \end{equation*}
    from which we see that $K_e$ satisfies the conditions of Theorem \ref{Thm1}. As the query space is $Q=S^d\subset B_2^d$, we can apply Theorem \ref{Thm1} with $R=1$. 
\end{proof}


We now argue that the same bound applies to the Hellinger kernel.
\begin{lemma}\label{hellinger}
    The Hellinger kernel $K_H$ has discrepancy at most that of the exponential kernel $K_e$, i.e.
    \begin{equation*}
        \mathrm{disc}_{K_H}(n)\leq \mathrm{disc}_{K_e}(n).
    \end{equation*}
\end{lemma}

\begin{proof}[Proof of Lemma \ref{hellinger}]
Given a data set $X\subseteq \Delta_d$ of size $n$, we make the transformation
\begin{equation*}
 f: \mathbb{R}^{d}\rightarrow\mathbb{R}^{d},\ \ \     (x_1,\ldots,x_{d})\rightarrow (\sqrt{x_1},\ldots,\sqrt{x_{d}}).
\end{equation*}
Here we make two observations: first, $f$ maps $\Delta^d$ into $S^d$, and second, for $x,y\in S^d$,
and
\begin{equation*}
  K_H(x,y)= \exp\Big(-\alpha\sum_{i\in[d]} (\sqrt{x_i}-\sqrt{y_i})^2\Big)=\exp(-\alpha\|f(x)-f(y)\|_2^2)=K_e(f(x),f(y)).
\end{equation*}

By the first observation, we can apply our bound on the exponential kernel discrepancy of the set $\{f(x):\ x\in X\}\subseteq S^d$ to find signs $\beta\in \{\pm 1\}^X$ so that
\begin{align*}
  \sup_{y\in S^d} \  \Big|\sum_{x\in X}\beta(x)K_e(f(x), y) \Big| &= \sup_{z\in \Delta^d} \  \Big|\sum_{x\in X}\beta(x)K_e(f(x), f(z)) \Big| \\
  &=\sup_{z\in \Delta^d} \Big| \sum_{x\in X}\beta(x)K_H(x,z) \Big| \\
  &=\mathrm{disc}_{K_H}(X) 
\end{align*}
satisfies the bound of Theorem \ref{expthm}. Thus we conclude that
\begin{equation*}
    \mathrm{disc}_{K_H}(X)\leq \mathrm{disc}_{K_e}(X),
\end{equation*}
and as $X$ was an arbitrary data set, the lemma follows.


\end{proof}

Theorem \ref{hell} then follows immediately from Theorem \ref{expthm} and Lemma \ref{hellinger}.
\section{The Jensen-Shannon kernel}\label{jssection}

Finally, we prove the following bound for a specific kernel that does not directly satisfy the conditions of Theorems \ref{Thm1} and \ref{Thm2}: the Jensen-Shannon (JS) Kernel, defined by
\begin{equation*}
        K_{JS}:\Delta_d\times \Delta_d\rightarrow \mathbb{R}\ \ \ \ K_{JS}(x,y)=\exp\left[-\alpha\left(H\left(\frac{x+y}{2}\right)-\frac{H(x)+H(y)}{2}\right)\right],
    \end{equation*}
    where $H(z)=-\sum_{i\in[d]}z_i\ln z_i$ is the Shannon entropy function. Note that the quantity in the exponent of the JS Kernel, $H\left(\frac{x+y}{2}\right)-\frac{H(x)+H(y)}{2}$, is precisely the Jensen gap for the Shannon entropy function (hence the name Jensen-Shannon kernel).

The Jensen-Shannon kernel, as well as the Hellinger kernel introduced in Section \ref{applications}, are commonly used in information theory, as they represent similarity between histograms (i.e. the data points are discrete probability distributions over $d$ variables). One example from text analysis \cite{phillipstainearoptimal} is encoding with each data point $X$, corresponding to a text, the probabilities of $d$ certain words appearing in that text. Although the JS kernel does not satisfy many of the conditions of our theorems, a similar proof strategy yields Theorem \ref{jsthm}.

 To prove this result, we will need the following lemmas bounding the one-dimensional entropy function, which we will denote by $h(x) := x \ln(\frac{1}{x}) = -x\ln(x)$. 
\begin{lemma} \label{lem:EntropyMidPointBoundI}
For $a,\delta \geq 0$ one has $|2h(a + \delta)- (h(a)+h(a+2\delta))| \leq 3\delta$.
\end{lemma}
\begin{proof}
  By continuity of $h$ it suffices to prove the inequality for the case of $a>0$.
  We have
  \begin{eqnarray*}
   & &  |2h(a+\delta)-h(a)-h(a+2\delta)| \\
    &=& | 2(a+\delta) \ln(a+\delta)-a \ln(a) - (a+2\delta)\ln(a+2\delta)| \\
                                     &=& \Big|2(a+\delta) \cdot \Big(\ln(a)+\ln\Big(1+\frac{\delta}{a}\Big)\Big) - a\ln(a) - (a+2\delta) \cdot \Big(\ln(a) + \ln\Big(1+\frac{2\delta}{a}\Big)\Big)\Big| \\
                                     &=& \Big|2(a+\delta) \ln\Big(1+\frac{\delta}{a}\Big) - (a+2\delta) \ln\Big(1+\frac{2\delta}{a}\Big) \Big| \\ 
                                           &\stackrel{\textrm{triangle ineq.}}{\leq}& a \cdot \underbrace{\Big|2\ln\Big(1+\frac{\delta}{a}\Big)-\ln\Big(1+\frac{2\delta}{a}\Big)\Big|}_{\leq \delta/a} + 2\delta \cdot \underbrace{\Big|\ln\Big(1+\frac{\delta}{a}\Big)-\ln\Big(1+\frac{2\delta}{a}\Big)\Big|}_{\leq 1} \leq 3\delta
  \end{eqnarray*}
  Here one can verify that $|2\ln(1+z)-\ln(1+2z)| \leq z$ for all $z \geq 0$ which we use to bound the left term.
  To bound the right term we use that   $|\ln(1+z)-\ln(1+2z)| \leq 1$ for all $z \geq 0$.
\end{proof}

\begin{lemma} \label{lem:EntropyMidPointBoundII}
For $a,b \geq 0$ one has $|h(\frac{a+b}{2})-\frac{h(a)+h(b)}{2}| \leq \frac{3}{4}|a-b|$. 
\end{lemma}
\begin{proof}
  Suppose $b \geq a$.
  Applying Lemma~\ref{lem:EntropyMidPointBoundI} with $a$ and $\delta := \frac{b-a}{2}$ gives
  \[
   \Big|2 h\Big(\frac{a+b}{2}\Big) - (h(a) + h(b))\Big| \leq 3 \cdot \frac{b-a}{2} = \tfrac{3}{2}|a-b|.
  \]
  The claim follows after dividing both sides by 2.
\end{proof}
\begin{proof}[Proof of Theorem \ref{jsthm}]
As the JS Kernel is positive definite, there exists a map $\phi:\Delta^d\rightarrow \mathcal{H}_{K_{JS}}$, where $\mathcal{H}_{K_{JS}}$ is a RKHS for $K_{JS}$, i.e. for any choice of $x,y\in \Delta_d$, $K_{JS}(x,y)=\langle \phi(x),\phi(y)\rangle_{\mathcal{H}_{K_{JS}}}$. Then as $\|\phi(x)\|_{\mathcal{H}_{K_{JS}}}=e^0=1$ for any choice of $x\in \Delta^d$, we can apply the Gram-Schmidt walk to the collection of vectors $\{\phi(x)\}_{x\in \Delta_d}$ exactly as in the proofs of Theorems~\ref{Thm1} and~\ref{Thm2} to obtain a distribution $\mathcal{D}$ over $\{\pm 1\}^X$ and a corresponding family of $O(1)$-subgaussian random variables
\begin{equation*}
    \Sigma_y:=\mathrm{disc}_{K_{JS}}(X, \beta,y),\ \ \ y\in \Delta_d,\ \beta\sim \mathcal{D}.
\end{equation*}
Then applying Theorem \ref{dud} to this collection of random variables with the pseudometric
\begin{equation*}
    D_{K_{JS}}(x,y)=\|\phi(x)-\phi(y)\|_{\mathcal{H}_{K_{JS}}}=\sqrt{2-2K_{JS}(x,y)},
\end{equation*}
we find that
\begin{equation}\label{dudjs}
    \mathbb{E} \ \mathrm{disc}_{K_{JS}}(X)\lesssim\int_0^{\mathrm{diam}(D_{K_{JS}})}\sqrt{\log \mathcal{N}(\Delta_d,D_{K_{JS}},r)}\ \mathrm{d}r.
\end{equation}

We focus on bounding the term $\mathcal{N}(\Delta_d,D_{K_{JS}}, r)$, for which we will use Lemma~\ref{lem:EntropyMidPointBoundII}. First, note that for $x=(x_1,\ldots,x_d)\in \mathbb{R}^d$, breaking the $d$-dimensional entropy function $H$ down component-wise as $H(x)=\sum_{i\in[d]}h(x_i)$:
\begin{equation}\label{bound}
    D_{K_{JS}}(x,y)\leq r\ \ \iff\ \ \sum_{i\in[d]} \left(h\left(\frac{x_i+y_i}{2}\right)-\frac{h(x_i)+h(y_i)}{2}\right)\leq \frac{1}{\alpha}\log \big(\tfrac{2}{2-r^2}\big).
\end{equation} 

By Lemma \ref{lem:EntropyMidPointBoundII}, we know that for each $i\in[d]$,
\begin{equation*}
    h\left(\frac{x_i+y_i}{2}\right)-\frac{h(x_i)+h(y_i)}{2}\leq |x_i-y_i|,
\end{equation*}
hence
\begin{equation*}
 \sum_{i\in[d]} \left(h\left(\frac{x_i+y_i}{2}\right)-\frac{h(x_i)+h(y_i)}{2}\right)\leq \sum_{i\in[d]}|x_i-y_i|=\|x-y\|_1.   
\end{equation*}

In particular, by \eqref{bound}, $\|x-y\|_1\leq \frac{1}{\alpha}\log \big(\tfrac{2}{2-r^2}\big) =: c$ implies that $D_{K_{JS}}(x,y)\leq r$, so we have the bound
\begin{equation*}
    \mathcal{N}(\Delta_d, D_{K_{JS}}, r)\leq \mathcal{N}(\Delta_d, \|\cdot\|_1, c ) \leq N(B_1^d, c B_1^d)  \leq \Big(1+\frac{2}{c}\Big)^d \leq \Big(\frac{3}{c}\Big)^d 
\end{equation*}
for $c \leq 1$ using Lemma~\ref{coveringOfConvexBody} (and if $c >1$, the covering number is 1 anyway). 
Returning to \eqref{dudjs}, we conclude
\begin{equation*}
\mathbb{E} \ \mathrm{disc}_{K_{JS}}(X)\lesssim\int_0^{\mathrm{diam}(D_K)}\sqrt{\log\left[\left(\frac{3\alpha}{\log\tfrac{2}{2-r^2}}\right)^d\right]}\ \mathrm{d}r   \lesssim \sqrt{d\log (2\max\{1,\alpha\})},
\end{equation*}
by the same argument as in the proofs of Theorems \ref{Thm1} and \ref{Thm2}.
\end{proof}

\section{Conclusions and Potential Future Questions}

In this paper we give improved bounds of order $O(\frac{\sqrt{d}}{\varepsilon}\sqrt{\log\log \tfrac{1}{\varepsilon}})$ on the coreset complexity of the Gaussian and Laplacian kernels (as well as other kernels satisfying sufficiently nice properties) in the case that the data set of interest has $O(1)$ diameter. We also give improved bounds of order $O(\frac{1}{\varepsilon}\sqrt{\log\log (1/\varepsilon)})$ for the Laplacian kernel in the case that the dimension $d$ is constant, extending a recent result of Tai for the Gaussian kernel \cite{taioptimal} with a $\sqrt{\log\log (1/\varepsilon)}$ loss. Finally, we provide the best known bounds of $O(\frac{\sqrt{d}}{\varepsilon}\sqrt{\log(2\max\{\alpha,1\})})$ for the Jensen-Shannon, exponential, and Hellinger kernels, in particular significantly improving the dependence on the bandwidth parameter $1/\alpha$ for the exponential kernel. To conclude the paper, we echo two remaining open question of interest given also in \cite{phillipstainearoptimal,taioptimal} respectively. First, it would be interesting to prove (or disprove) an upper bound of $O(\sqrt{d}/\varepsilon)$ on the coreset complexity of Lipschitz positive definite kernels with nice decay properties to match the bound in \cite{phillipstainearoptimal}. Second, it would be interesting to extend our $O(\sqrt{\log\log n})$ bound  from Theorem \ref{Thm2} in constant dimension to either an even broader class of kernels, or even more interestingly, to drop the $\sqrt{\log\log n}$ dependence entirely, as Tai is able to do in the special case of the Gaussian. Our analytic method allows us to move one step closer to either of these results by removing the extra factor of $\sqrt{\log n}$ introduced by discretizing the query space.

\bibliography{Bibliography}
\bibliographystyle{plain}
\end{document}